\def\year{2021}\relax
\documentclass[letterpaper]{article} % DO NOT CHANGE THIS
\usepackage{aaai21}  % DO NOT CHANGE THIS
\usepackage{times}  % DO NOT CHANGE THIS
\usepackage{helvet} % DO NOT CHANGE THIS
\usepackage{courier}  % DO NOT CHANGE THIS
\usepackage[hyphens]{url}  % DO NOT CHANGE THIS
\usepackage{graphicx} % DO NOT CHANGE THIS
\usepackage{natbib}  % DO NOT CHANGE THIS AND DO NOT ADD ANY OPTIONS TO IT
\usepackage{caption} % DO NOT CHANGE THIS AND DO NOT ADD ANY OPTIONS TO IT
\nocopyright
\usepackage{subfigure}
\usepackage{array}
\usepackage{amsthm}
\usepackage{amsmath}
\usepackage{amssymb}

\makeatletter
\newcommand{\thickhline}{%
    \noalign {\ifnum 0=`}\fi \hrule height 1pt
    \futurelet \reserved@a \@xhline
}

\newcolumntype{"}{@{\hskip\tabcolsep\vrule width 1pt\hskip\tabcolsep}}
\makeatother

\def \x {\mathbf{x}}
\def \s {\mathbf{s}}
\def \y {\mathbf{y}}

\def \y {\mathbf{y}}

\def \z {\mathbf{z}}

\def \L {\mathcal{L}}

\newtheorem{prop}{Proposition}

\title{Semi-Anchored Detector for One-Stage Object Detection}
\author{
Lei Chen\quad Qi Qian\quad Hao Li\\
}
\affiliations{
Alibaba Group\\
\{fanjiang.cl, qi.qian, lihao.lh\}@alibaba-inc.com
}
\begin{document}

\maketitle

\begin{abstract}
A standard one-stage detector is comprised of two tasks: classification and regression. Anchors of different shapes are introduced for each location in the feature map to mitigate the challenge of regression for multi-scale objects. However, the performance of classification can degrade due to the highly class-imbalanced problem in anchors. Recently, many anchor-free algorithms have been proposed to classify locations directly. The anchor-free strategy benefits the classification task but can lead to sup-optimum for the regression task due to the lack of prior bounding boxes. In this work, we propose a semi-anchored framework. Concretely, we identify positive locations in classification, and associate multiple anchors to the positive locations in regression. With ResNet-101 as the backbone, the proposed semi-anchored detector achieves $43.6\%$ mAP on COCO data set, which demonstrates the state-of-art performance among one-stage detectors.
\end{abstract}

\section{Introduction}
%introduce one-stage detector
With the development of deep learning, object detection becomes more applicable for real-world applications using deep neural networks. Many modern detectors work in either one-stage or two-stage manners. In a two-stage detection pipeline, a region proposal method is adopted to eliminate most of background bounding boxes. After that, the remaining candidates will be refined in the second stage~\cite{HeGDG17,RenHGS15}. Recently, one-stage object detectors have attracted much attention due to its efficiency~\cite{LinGGHD17,LiuAESRFB16,RedmonF17}.  Different from two-stage detectors, one-stage detectors identify objects from all candidates directly without region proposal. The compact architecture makes one-stage detectors appropriate for mobile devices with limited computing resources.

\begin{figure}[!ht]
\centering
\includegraphics[width=3.2in]{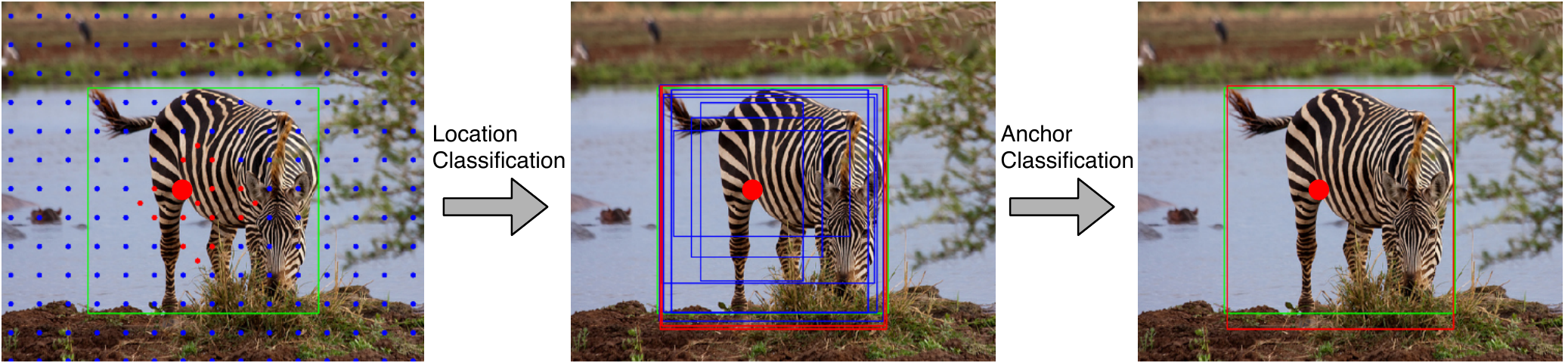}
\caption{An illustration of the proposed semi-anchored detector. It eliminates most of negative locations that contain little foreground anchors with an anchor-free strategy (i.e., location classification). An anchor classifier is further applied to predict labels of anchors at each positive location. (The ground-truth bounding box is highlighted in green. The red and blue bounding boxes are positive and negative ones, respectively)\label{fig:pipeline}}
\end{figure}

%introduce anchor
Object detection can be comprised of two tasks, that is, classification and regression. Classification is to obtain the candidate locations or bounding boxes for foreground objects while regression is to refine the corresponding bounding boxes. Many existing detectors apply anchors as candidate bounding boxes~\cite{LinGGHD17,RedmonF17,RenHGS15}. Anchor is introduced in a two-stage detector, i.e., faster R-CNN~\cite{RenHGS15}, to handle multiple scales of objects. For each location in a feature map, multiple prior bounding boxes with different scales and aspect ratios are associated with it, which are defined as anchors. This strategy helps to approach ground-truth bounding boxes for different objects by varying the scales of anchors in regression. It achieves tremendous success for object detection but makes the issue of class-imbalance in classification more challenging. 

%introduce class-imbalance
Class-imbalance problem is prevalent in object detection since the number of background candidates can be significantly larger than that of foreground ones. When adopting anchors, the ratio of imbalance becomes more challenging. It is because that each location has multiple anchors due to different combinations of scales and aspect ratios. Consequently, more background bounding boxes can be produced even for the positive locations. This issue is often handled by a cascade pipeline in two-stage detectors, where the region proposal phase in the first stage only selects a small amount of anchors as candidates for the second stage. In contrast, one-stage detectors have to deal with all anchors simultaneously. Various strategies have been developed to mitigate this challenge for one-stage detectors. For example, SSD~\cite{LiuAESRFB16} samples hard background bounding boxes for training, which is similar to region proposal methods. RetinaNet~\cite{LinGGHD17} proposes focal loss to reduce the influence from the massive number of background anchors. 

%introduce anchor free
Unlike the conventional pipeline with anchors, certain algorithms consider to detect objects without anchors, which can be categorized as anchor-free algorithms~\cite{tian2019fcos}. These algorithms classify locations in the feature map directly and then obtain the bounding boxes centered at the foreground pixels by regression. Without anchors in classification, the performance of identifying positive locations can be improved since the number of candidates from foreground and background is more balanced. However, the regression task becomes challenging since it has to predict the shapes of bounding boxes solely from center pixels rather than a set of pre-defined anchors. The degraded performance on regression can reduce the gain from the classification for object detection.

%introduce semi-anchored
In this paper, we propose a semi-anchored detector for one-stage object detection. Specifically, we classify locations in the feature map without anchors for the classification task. We can improve the ratio of positive/negative candidates from $1:1400$ with anchors to $1:200$ with locations in classification. For regression, we associate multiple anchors for each location and learn the bounding boxes from anchors for foreground locations. One of the main challenges is to compute the foreground/background probabilities for anchors centered at foreground locations, since anchors were not used in classification. Those anchors can share the probabilities of the corresponding locations, but the performance can be sub-optimal due to the different shapes of anchors. Therefore, we attach an anchor classification head to identify foreground anchors for each location. Fig.~\ref{fig:pipeline} illustrates the procedure of the proposed semi-anchored detector.

Apparently, the proposed detector handles a balanced classification problem without anchors and obtains an enhanced performance of regression with anchors. Besides, we define the positive anchors according to the intersection over union (IoU) after regression. Compared with the conventional algorithms, which label anchors with IoU before regression, the proposed strategy is more consistent with the target. Moreover, the efficiency can be improved due to the simplified classification head. The extensive experiment on COCO data set~\cite{LinMBHPRDZ14} verifies the effectiveness and efficiency of the proposed framework. Our algorithm can surpass FCOS~\cite{tian2019fcos}, which is a state-of-the-art anchor-free detector, and achieve mAP $43.6\%$ with ResNet-101 as the backbone. Furthermore, the inference time of the proposed method is less than RetinaNet even when we assign more anchors for each location.

%-------------Organization
%The rest of the paper is organized as follows: Section~\ref{sec:related} summarizes the related work on object detection. Section~\ref{sec:method} describes the proposed semi-anchored detector. Section~\ref{sec:exp} illustrates the results of the empirical study and Section~\ref{sec:conclusion} concludes this work with discussions.

\section{Related Work}
\label{sec:related}
\paragraph{Two-stage Detectors.} 
Many conventional object detectors have two stages. In the first stage, a small set of candidate bounding boxes that probably contain objects are proposed. Then, those candidates can be refined in the second stage. Explicitly, two-stage detectors work in a cascade manner. Since the first stage filters most of background candidates, the final problem in the second stage is well balanced and can be addressed well. With the development of deep neural networks, two-stage detectors demonstrate a superior performance on benchmark data sets as follows. 

For example, R-CNN~\cite{GirshickDDM14} applies Selective Search~\cite{UijlingsSGS13} to generate candidate proposals and classifies candidates with features from the convolutional neural networks (CNNs) in the second stage. Fast R-CNN~\cite{Girshick15} improves the efficiency of feature extraction from CNNs. Faster R-CNN~\cite{RenHGS15} introduces CNNs for the first stage and proposes the region proposals network (RPN) to obtain the candidates, which further reduces the computational cost. Moreover, many variants of R-CNN have been proposed~\cite{CaiV18,DaiLHS16,abs-1906-02739,HeGDG17,LinDGHHB17,abs-1811-12030,PangCSFOL19}. However, those detectors consists of two stages for inference, which is inefficient for applications with limited resources.

\paragraph{One-stage Detectors.} 
To simplify the architecture of two-stage detectors for real-world applications, researchers try to detect objects with a single stage~\cite{LinGGHD17,LiuAESRFB16,RedmonF17,tian2019fcos,abs-1901-03278,YangZLZS18,abs-1909-02466,abs-1903-00621,QianCLJ20}. Since anchors are prevalent in two-stage detectors, many one-stage detectors also work on anchors~\cite{LinGGHD17,LiuAESRFB16,RedmonF17}. To mitigate the imbalance problem in anchors, SSD~\cite{LiuAESRFB16} applies hard example mining to select anchors for training. Furthermore, RetinaNet~\cite{LinGGHD17} proposes focal loss to reduce the influence from the massive background anchors. Besides, some works consider to optimize the shape of anchors. YOLOv2~\cite{RedmonF17} adopts clustering to make sure that the initial shape of anchors can approximate the ground-truth bounding boxes well. MetaAnchor~\cite{YangZLZS18} and Guided Anchoring~\cite{abs-1901-03278} learn the shape of anchors within the training pipeline. All of these methods keep anchors for classification.

Recently, the anchor-free detector is proposed to eliminate anchors in one-stage detectors~\cite{tian2019fcos,ZhangCYLL20}. Without the additional negative examples introduced by anchors, the classification task can be solved more effectively. However, for regression, the algorithm has to predict the bounding boxes from the corresponding centers. Compared with regression with anchors, the task becomes more challenging and the performance of the detector can be sub-optimal. In this work, we propose a semi-anchored detector to overcome the class-imbalance problem in classification and take the benefit from anchors for regression. Finally, the issue that labels of anchors are computed with IoU before regression has attracted attentions in some works\cite{CaoPHL19,VuJPY19,06563}. We propose a simple strategy to mitigate the inconsistency by generating labels of anchors in positive locations with IoU after regression.

\section{Semi-Anchored Detector}
\label{sec:method}
Since anchors can result in the severe imbalance problem to the classification task, we propose to do anchor-free classification at first. Then, considering the significant performance improvement that anchors giving to regression, we propose to include anchors in regression. However, there will lead to a big gap between the classification and the regression task. 

Concretely, the anchor-free classification (i.e., location classification) task predicts the probability per \textit{location}, while the regression task improves the bounding boxes based on \textit{anchors}. Therefore, there is a lack of the probabilities of classes for anchors. If anchors from the same location share the same probability, it becomes hard to distinguish the best bounding box among them according to the pipeline of non-maximum suppression (NMS). Therefore, we propose a semi-anchored detector that uses a new location labeling strategy as described in Fig.~\ref{fig:mapping} and one more head to identify positive anchors for each location as elaborated in Fig.~\ref{fig:archi} to bridge the gap. Note that we optimize a standard IoU loss~\cite{YuJWCH16} for the regression head as suggested in \cite{tian2019fcos}, and we will focus on elaborating the classification task in this section.

\subsection{Location Classification}\label{subsec:labeling}

In the classification task, we identify foreground locations from the feature map without anchors. A location refers to a pixel in the feature map. Let $\{\x_i,y_i\}$ denote the set of locations, where $\x_i$ is the feature and $y_i$ indicates the label of the $i$-th location. For a problem with $C$ foreground objects, we let $y_i\in\{0,\dots,C\}$ where $y_i=0$ indicates a background location. Note that $\x_i$ can be extracted from the feature map directly and the only problem is to assign appropriate labels for locations.

The most straightforward way to label each location is using ground truth bounding boxes, that is, each location within a ground truth bounding boxes can be labeled by the corresponding foreground label. However, each location can be associated with multiple foreground objects. A heuristic method that labels an overlapped location with the label from a smaller object~\cite{tian2019fcos} may not be consistent with a regression task using anchors. Therefore, we propose to define the label for each location with anchors in the proposed semi-anchored detector.

For each location, we associate $K$ anchors with different scales and aspect ratios as in the anchor-based methods. Following the conventional algorithms, the label of anchors can be obtained by computing IoU with the ground-truth bounding boxes. Let the one-hot vector $\y_{i,k}\in\{0,1\}^{C+1}$ denote the label of the $k$-th anchor in the $i$-th location. With the labels of anchors for a location, we can obtain the confidence score for the location as
\[\s_i = \sum_k \y_{i,k}/K \]
Considering the significant large number of backgrounds in anchors, we re-scale the score with a constant $0\le\gamma\le 1$ to the backgrounds (i.e., $c=0$ where $\s_i^c$ is the $c$-th element in $\s_i$) as
\begin{eqnarray}\label{eq:mv}
\hat{\s}_i^c =\left\{\begin{array}{cc}\gamma \s_i^c&c=0\\(1-\gamma)\s_i^c&o.w.\end{array}\right.
\end{eqnarray}
which is equivalent to threshold moving to address a class-imbalance problem.

\begin{figure}[t]
\centering
\includegraphics[width=3.2in]{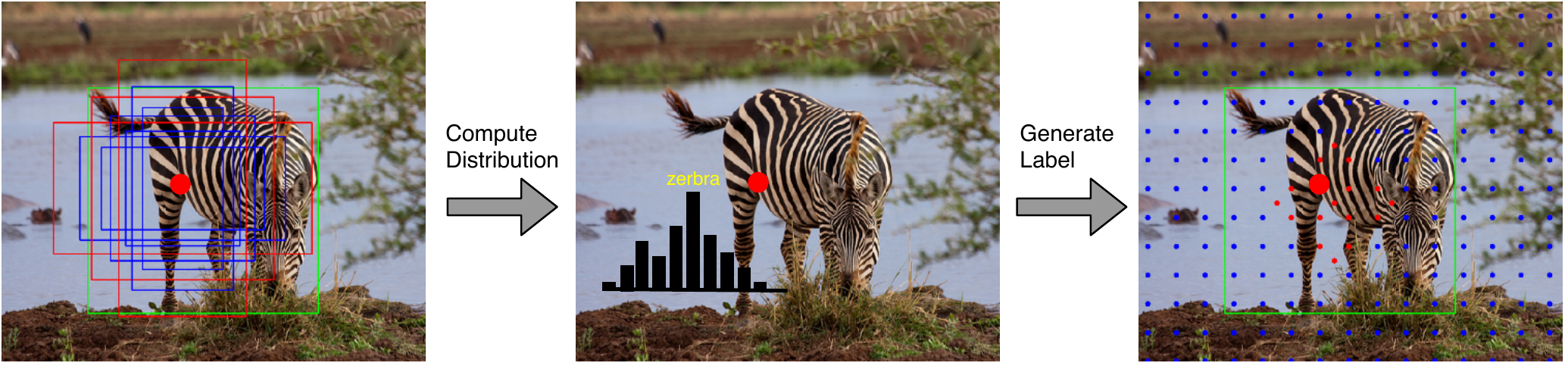}
\caption{An illustration of labeling strategy for locations. It first collects labels from associated anchors. Then, the label distribution for each location is computed accordingly. Finally, the most confident label will be assigned to the location. (The ground-truth bounding box, positive and negative bounding box are highlighted in green, red and blue, respectively. The dots indicate locations in the feature map. The red ones are labeled as foreground while the blue ones are for background.)\label{fig:mapping}}
\end{figure}

Given the confidence score, the label of the $i$-th location can be defined as
\[y_i = \arg\max_c \{\hat{\s}_i^c\}\]
which highly depends on $\gamma$. When $\gamma$ is sufficiently small, a location with any positive anchor can be labeled as positive, which is proved in the following proposition.

\begin{prop}
According to Eqn.~\ref{eq:mv}, when $\gamma < 1/K$, the label of the $i$-th location $y_i$ will be positive if any positive anchor is associated with the $i$-th location.
\end{prop}
\begin{proof}
First, we normalize the score $\hat{\s}_i$ to indicate the label distribution for each location as $\tilde{\s}_i^c = \hat{\s}_i^c/\sum_c\hat{\s}_i^c$, which can be used to demonstrate the confidence of a selected label. 
Then, assuming the $j$-th foreground label has $n_j$ anchors associated with the $i$-th location and $\sum_{j=1}^{C}n_j\ge 1$, we have
\begin{align*}
&\Pr\{y_i = c\} =\tilde{\s}_i^c> \frac{n_c}{\sum_{j=1}^{C}n_j +\frac{K-\sum_{j=1}^C n_j}{K-1}}\\
&\ge \frac{n_c}{1+\sum_{j=1}^{C}n_j}
\end{align*}

where the first inequality is from the assumption that $\gamma < 1/K$ and the second inequality is due to $\sum_{j=1}^{C}n_j\ge 1$.

With the similar analysis, we have the background probability as
\begin{align*}
&\Pr\{y_i =0 \}=\tilde{\s}_i^0<\frac{n_0}{n_0+(K-1)\sum_{j=1}^C n_j}\\
&=\frac{1}{1+\sum_{j=1}^{C}n_j\frac{K-1}{K-\sum_{j=1}^{C}n_j}}\leq \frac{1}{1+\sum_{j=1}^{C}n_j}
\end{align*}
Therefore, if there exists any foreground anchor that $\exists j\in\{1,\dots,C\}, n_j\geq 1$ associating with the $i$-th location, the location will be labeled as positive with a confidence larger than $\max_c \{\frac{n_c}{1+\sum_{j=1}^{C}n_j}\}$.
\end{proof}

It should be noted that anchors are only used to obtain labels for locations in the classification task. The proposed labeling strategy is illustrated in Fig.~\ref{fig:mapping}. It can be observed that many non-essential locations within the ground-truth bounding boxes are labeled as background with the proposed pipeline. To further demonstrate our labeling strategy, we show some examples of positive locations in Fig.~\ref{fig:label}. We can observe that the locations with positive labels are concentrated at the centers of the ground-truth bounding boxes.

\begin{figure}[t]
\centering
\includegraphics[width=3.2in]{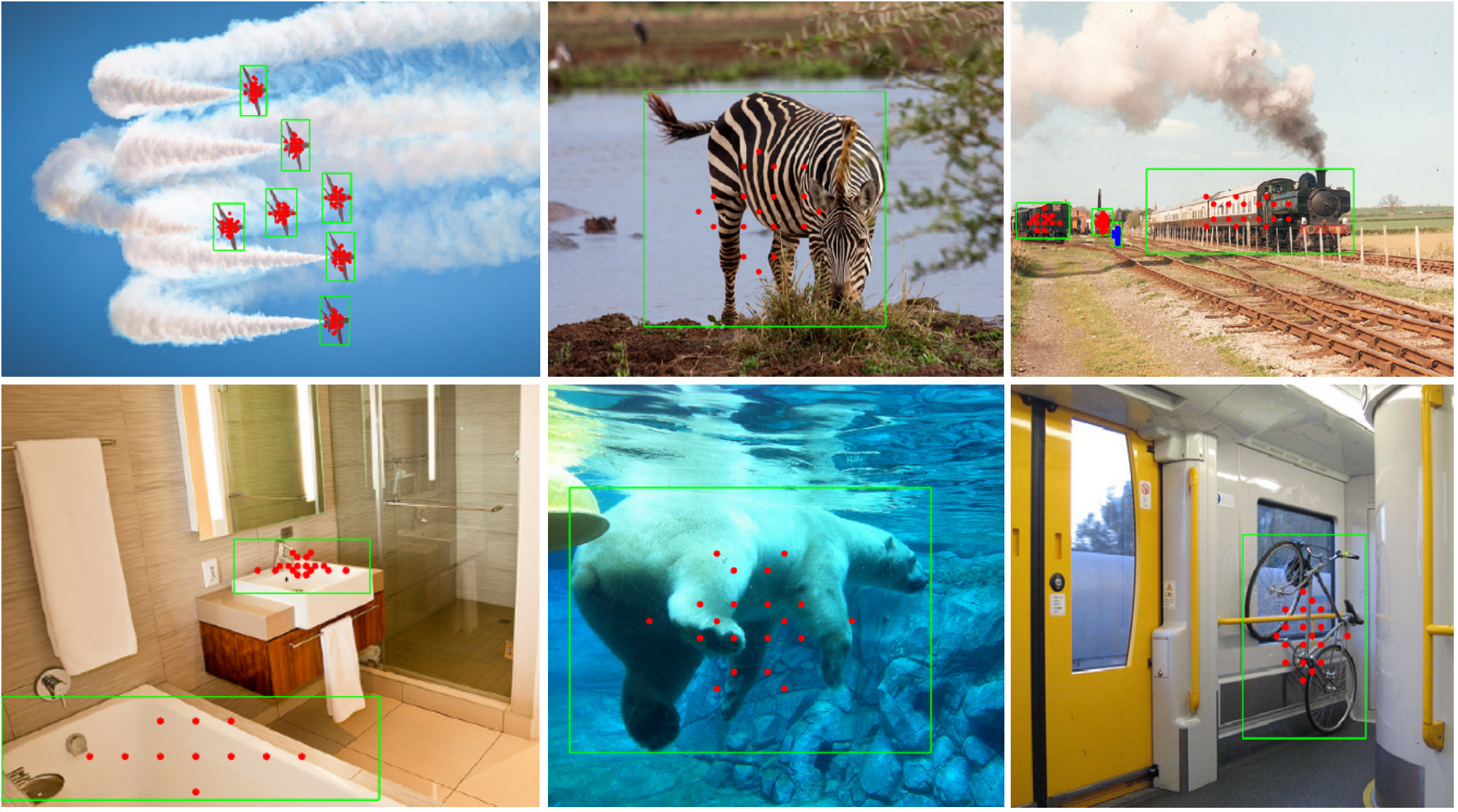}
\caption{Illustration of positive locations from different images in COCO data sets. The green bounding box denotes the ground-truth bounding box. The red dots illustrate the positive locations. Note that the locations are asymmetric due to the different strides of features.\label{fig:label}}
\end{figure}

With the labeled locations, we can train the classifier with the popular focal loss as in other works~\cite{LinGGHD17,tian2019fcos}
\[\L_{\mathrm{cls}} = \frac{1}{\sum_i \mathbb{I}_{y_i>0}} {\sum_i \mathrm{FL}(p_i, y_i)}\]
where $\mathbb{I}(\cdot)$ is the indicator function and $p_i = \Pr\{y_i | \x_i\}$ is the prediction. The classification loss is accumulated over all locations. Apparently, we can get rid of the serious imbalance problem in the classification task. Then, we can do regression. However, when using anchors after regression, we still lack the probability information of useful anchors. Therefore, we propose to include one more head to identify positive anchors for each location as described in the next subsection. 

\subsection{Anchor Classification}\label{subsec:head}

Now given the location classifier and anchor regressor, during the inference, the location classifier may tell that the probability $\Pr\{y_i=c|\x_i\}$ of the $i$-th location for object $c$ and the regressor provides $K$ bounding boxes from anchors $\{\z_{i,k}\}_{k=1,\dots,K}$, where $\z_{i,k}$ denotes features for $k$-th anchor at the $i$-th location. Here comes the main gap, that is, which of the $K$ anchors should be the output. Therefore, we aim to estimate the probability for each anchor as $\Pr\{y_{i,k}=c|\x_i,\z_{i,k}\}$ while only the probabilities for the corresponding locations $\Pr\{y_i=c|\x_i\}$ is available.

Considering that the label of anchors should be consistent with its location, we compute the conditional probability as
\begin{align*}
&\Pr\{y_{i,k}=c|\x_i,\z_{i,k}\} \\
&= \Pr\{y_i=c|\x_i\} \Pr\{y_{i,k}=y_i|\z_{i,k},y_i=c\}
\end{align*}

The formulation implies a binary classification problem that identifies the anchors with the same label as the locations. Therefore, we can collect the training set as $\{\z_{i,k},\hat{y}_{i,k}\}$, where 
\[\hat{y}_{i,k} = \left\{\begin{array}{ll}1&\quad y_{i,k}=y_i\\0 &\quad y_{i,k}\not=y_i\end{array}\right.\]

In the conventional anchor-based methods, the labels of anchors $\{y_{i,k}\}$ are computed according to the prior shapes of anchors. After regression, the refined shapes can be different from the initial ones, which actually leads to a big disparity.
Our proposal can remove this disparity by computing the IoU for the improved anchors after regression. With the appropriate labels, we can learn the anchor classifier by optimizing a focal loss. 

It should be noted that the labels of anchor classification are binary while the IoU is a continuous number, which means the optimal probability for different anchors can vary. Inspired by the knowledge distillation~\cite{HintonVD15}, we consider to incorporate soft labels to train the anchor classifier.

Let $\mu_{i,k}$ denote the IoU of anchor $\z_{i,k}$. First, we normalize the IoU scores for each location as
\begin{eqnarray}\label{eq:uk}
\hat{\mu}_{i,k} = (\mu_{i,k}/\max_k\{\mu_{i,k}\})^\sigma
\end{eqnarray}
where the optimal anchor at each location will have the score $1$ and the rest will reduce their scores based on the parameter $0<\sigma<1$. Then, we adopt the score as the soft label and introduce a smoothed focal loss
\begin{eqnarray}\label{eq:sfl}
&&\L_{s}(p_{i,k},\hat{\mu}_{i,k},\hat{y}_{i,k}) \\
&&= \left\{\begin{array}{cc} -\alpha (|\hat{\mu}_{i,k} - p_{i,k}|)^\beta \hat{\mu}_{i,k} \log(p_{i,k}) &\quad \hat{y}_{i,k} = 1\\ -(1 - \alpha) p_{i,k}^\beta \log(1-p_{i,k})&\quad o.w.\end{array}\right.\nonumber
\end{eqnarray}
where $p_{i,k} = \Pr\{\hat{y}_{i,k}=1|\z_{i,k}\}$ is the prediction of the anchor classifier. Compared with the standard focal loss, we have $\hat{\mu}_{i,k}$ as a smoothed label for the positive anchor instead of $1$, which can capture the distribution of different anchors better and improve the performance slightly as illustrated in Table~\ref{table:sfloss}.

The suggested configuration for $(\alpha,\beta)$ in \cite{LinGGHD17}, i.e., $(\alpha,\beta) = (0.25,2)$, is adopted for anchor classification while $(\alpha,\beta) = (0.25,1)$ is applied for the standard focal loss in location classification. With the proposed smoothed focal loss, the anchor classifier is learned by minimizing the loss over all foreground locations as
\[\L_{\mathrm{ac}} = \frac{1}{N_{\mathrm{anchor}}^+} \sum_{i} \mathbb{I}_{y_i>0}  \sum_{k} \L_{s}(p_{i,k},\hat{\mu}_{i,k},\hat{y}_{i,k})\]
where $N_{\mathrm{anchor}}^+$ indicates the total number of positive anchors.

In summary, the objective of the semi-anchored detector is to minimize
\[\L = \L_{\mathrm{cls}}+\lambda_{\mathrm{reg}}\L_{\mathrm{reg}}+\lambda_{\mathrm{ac}}\L_{\mathrm{ac}}\]
where we fix $\lambda_{reg} = 2$ and $\lambda_{ac}=1$ in this work. The architecture of the proposed semi-anchored detector is illustrated in Fig.~\ref{fig:archi}. We adopt the backbone of RetinaNet~\cite{LinGGHD17} and change only the head branches for classification and regression. Note that the anchor classification head shares features with the regression head due to the high correlation between them. Besides, the computational efficiency can be slightly improved without an additional branch for anchor classification.

\begin{figure}[t]
\centering
\includegraphics[width=3.2in]{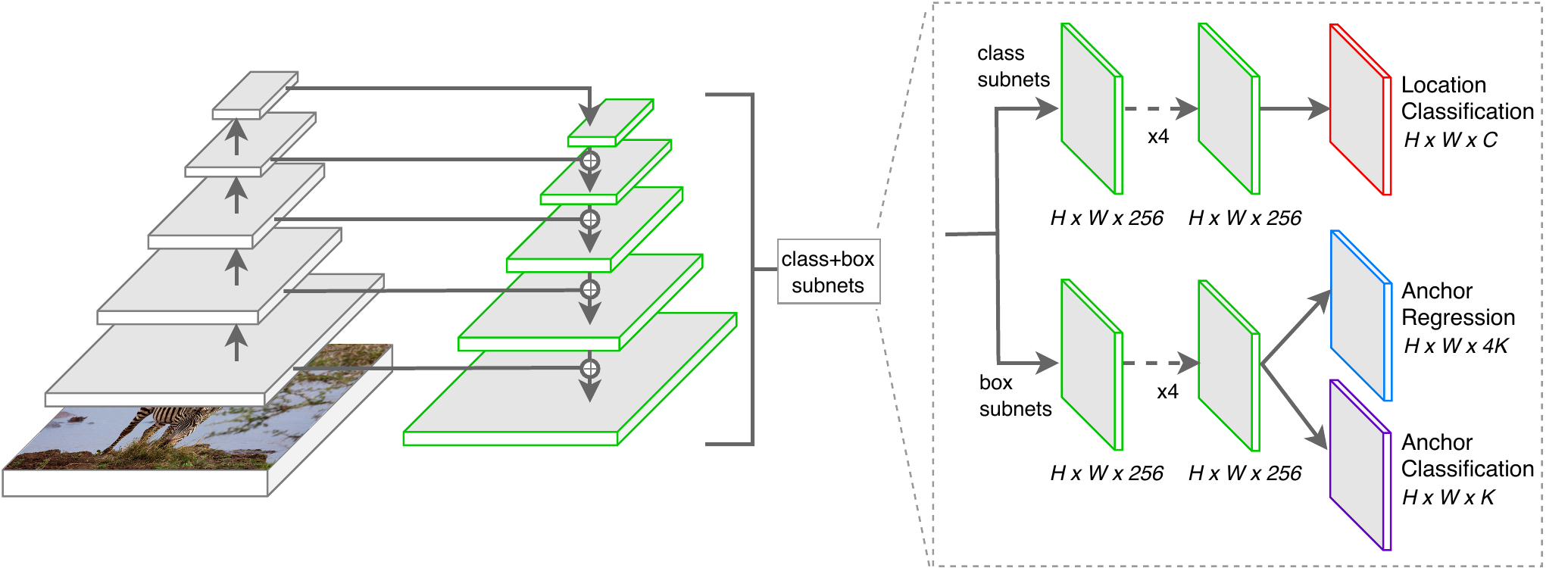}
\caption{An illustration of the architecture of the proposed framework. The backbone of RetinaNet is applied. We add an anchor classification head to identify foreground anchors at positive locations.\label{fig:archi}}
\end{figure}

\section{Experiments}
\label{sec:exp}

To evaluate the proposed method, we compare the semi-anchored detector to benchmark algorithms on MS-COCO 2017~\cite{LinMBHPRDZ14}. COCO training set has 118k images for training, 5k images for validation. The test set consists of 40k images. For ablation study, all models are trained for 90k iterations with an initial learning rate of 0.01 and a mini-batch of 16 images on 8 GPUs. The learning rate is decayed by a factor of 10 at 60k and 80k iterations, respectively. Horizontal image flipping is the only applied data augmentation unless otherwise specified. Weight decay and momentum in SGD optimizer are set to 0.0001 and 0.9, respectively. The parameters in backbone network are initialized from that trained on ImageNet~\cite{ILSVRC15}. We report the mAP on validation set for ablation study and that on \textit{test-dev} set for performance comparison. For a fair comparison, our method is implemented within a public codebase\footnote{https://github.com/facebookresearch/maskrcnn-benchmark}, which is adopted by many existing methods.

\subsection{Ablation Study}
\label{subsection:ablation}
In this subsection, we conduct experiments to study the behavior of the proposed detector. Following the common practice, we adopt ResNet-50~\cite{HeZRS16} with a Feature Pyramid Network (FPN)~\cite{LinDGHHB17} as the backbone for ablation study. Furthermore, we adopt the configuration in FCOS~\cite{tian2019fcos}, where Group Normalization(GN)~\cite{WuH18} is applied in the newly added convolutional layers except for the last prediction layers and P6 and P7 are produced by P5 rather than C5. Those improvements can increase mAP by $0.8\%$ as illustrated in \cite{LinDGHHB17}. All other settings remain unchanged as in the standard RetinaNet. A single image scale of 800 is used for training and test in this subsection.

\subsubsection{Number of Anchors}
First, We evaluate the influence from the number of anchors. We vary the number of scales from $1$ to $5$ and that of aspect ratios in $\{1,3,5\}$. Table~\ref{ta:num} summarizes the results with different number of anchors. 

\begin{table}[!ht]
\centering
\small
\caption{Comparison of different number of anchors. The last column indicates the inference time ($ms$). \#s and \#a denote the number of scales and aspects, respectively.}
\label{ta:num}
\begin{tabular}{ll"p{0.5cm}p{0.5cm}p{0.5cm}p{0.5cm}p{0.5cm}p{0.5cm}p{0.5cm}}
\#s &\#a &AP&AP$_{50}$&AP$_{75}$&AP$_{S}$&AP$_{M}$&AP$_{L}$&Time\\\thickhline
\multicolumn{2}{l"}{RetinaNet}&35.9&56.0&38.2&20.0&39.8&47.4\\
1&1&35.8&55.0&38.1&21.2&38.3&48.2&95\\
2&1&37.4&55.8&40.4&21.0&41.2&49.8&96\\
3&3&39.0&57.2&42.4&\textbf{22.7}&43.1&52.3&97\\
4&3&38.8&56.8&42.0&22.3&42.8&51.5&98\\
4&5&39.4&57.4&42.3&22.2&43.4&52.5&97\\
\textbf{5}&\textbf{5}&\textbf{39.6}&\textbf{57.4}&\textbf{43.1}&22.3&\textbf{43.8}&\textbf{53.0}&98
\end{tabular}
\label{table:density}
\end{table}

We can observe that our method with a single anchor already achieves the similar performance as RetinaNet, where we adopt the performance of RetinaNet reported in \cite{tian2019fcos}. It demonstrates that reducing the number of negative examples in the classification task by classifying locations directly can significantly boost the performance of one-stage detectors. It is also consistent with the observation in RetinaNet that class-imbalance issue degrades the performance of one-stage detection algorithms. 

Second, the performance of semi-anchored algorithm improves with the increasing number of anchors. It shows that a single anchor cannot handle the objects with multiple scales well and multiple anchors can depict the bounding boxes better. 

To further verify the effectiveness of the anchor classification head, we compare the proposed method to the variant without an anchor classifier. Since there is no predicted probability for a single anchor given its location, we randomly pick an anchor from each positive location and then assign the probability of the corresponding location to the anchor. For the baseline without AC head, we repeat the experiments for $10$ times and report the best result. The comparison is shown in Table~\ref{ta:ac}. Evidently, the anchor classification head, as an essential component in the proposed semi-anchored detector, can dramatically boost the performance by identifying positive anchors from negative ones. Even when there are only two anchors at each location, our method can outperform the one with random sampling by a large margin of about $6.8\%$. It confirms our claim that an anchor classification head helps to refine the anchors at each location.

\begin{table}[!ht]
\centering
\caption{Comparison of the proposed algorithm with or without anchor classification (AC) head.}
\label{ta:ac}\begin{tabular}{lll"lll}
\#s&\#a&w/AC&AP&AP$_{50}$&AP$_{75}$\\\thickhline
2&1&&30.6&49.5&31.8\\
2&1&\checkmark&\textbf{37.4}&\textbf{55.8}&\textbf{40.4}
\end{tabular}
\label{table:density}
\end{table}

Finally, we note that the efficiency affected by multiple anchors is tiny. It is because we only optimize anchors in the regression and anchor classification heads. By eliminating most of background locations with location classification, the number of remained locations are limited. Note that even with $25$ anchors, the proposed algorithm is faster than RetinaNet with $9$ anchors, which consumes more than $98ms$ per image for inference. Our inference time is also comparable to FCOS~\cite{tian2019fcos}, which is an anchor free method and costs $98ms$ for inferring each image. We will adopt $25$ anchors in the following experiments.

\subsubsection{Location Classification}
Different strategies can be used to generate labels for locations. We compare three different methods and the comparison is shown in Table~\ref{ta:labeling}. We first try the labeling strategy in FCOS~\cite{tian2019fcos}, which is denoted as ``FCOS''. The drawback of this strategy is that it assigns all inside locations to be positive, which can include the background locations. Therefore, we can shrink the ground-truth box to reduce the valid projection area and the improved strategy is referred as ``FCOS-Shrink''. The mAP increases $0.2\%$ over the original one in FCOS. More importantly, the proposed labeling method for the semi-anchored detector can further improve mAP from $38.0\%$ to $39.6\%$. It demonstrates the effectiveness of the proposed location labeling strategy using anchors.

\begin{table}[!ht]
\centering
\small
\caption{Comparison of labeling strategies for locations.}
\label{ta:labeling}
\begin{tabular}{l"p{0.5cm}p{0.5cm}p{0.5cm}p{0.5cm}p{0.5cm}p{0.5cm}}
labeling strategy&AP&AP$_{50}$&AP$_{75}$&AP$_{S}$&AP$_{M}$&AP$_{L}$\\\thickhline 
FCOS&37.8&55.0&40.7&21.7&42.0&49.9\\
FCOS-Shrink&38.0&55.8&41.0&21.9&41.9&49.9\\
Ours&\textbf{39.6}&\textbf{57.4}&\textbf{43.1}&\textbf{22.3}&\textbf{43.8}&\textbf{53.0}
\end{tabular}
\end{table}

Then, we study the threshold moving strategy in Eqn.~\ref{eq:mv}. Table~\ref{ta:gamma} shows the results when increasing $\gamma$ in Eqn.~\ref{eq:mv} from $1/26$ to $0.2$. When $\gamma = 1/26$, it is less than $1/25$, which means that a location with any foreground objects will be labeled as foreground. By increasing $\gamma$, a location should contains more foreground anchors to have a positive label. We observe that with a small $\gamma$, the performance of semi-anchored detector is significantly better than that with a large $\gamma$. It demonstrates that all locations with foreground anchors should be kept for the regression and anchor classification. This observation also helps to simplify our labeling strategy for locations. Given the confidence score of the $i$-th location $s_i$, the label can be obtained as \begin{alignat}{1}y_i = \left\{\begin{array}{cc} 0&\s_i^0=1\\\arg\max_{c\in\{1,\cdots,C\}}\{\s_i^c\}&o.w.\end{array}\right.\end{alignat} which gets rid of the parameter $\gamma$. This strategy will be adopted in the rest of experiments.

\begin{table}[!ht]
\centering
\caption{Comparison of varying $\gamma$ for labeling locations.}
\label{ta:gamma}
\begin{tabular}{l"llllll}$\gamma$&AP&AP$_{50}$&AP$_{75}$&AP$_{S}$&AP$_{M}$&AP$_{L}$\\\thickhline
1/26&\textbf{39.6}&\textbf{57.4}&\textbf{43.1}&\textbf{22.3}&\textbf{43.8}&\textbf{53.0}\\
0.05&39.1&56.9&42.6&22.1&43.6&52.9\\0.1&38.8&56.6&42.2&22.3&43.0&52.3\\
0.2&35.6&52.6&38.3&18.6&40.2&47.9
\end{tabular}
\label{table:labels}
\end{table}

\subsubsection{Anchor Classification}
Here, we demonstrate the effectiveness of the proposed smoothed focal loss in Eqn.~\ref{eq:sfl} for anchor classification. The comparison is summarized in Table~\ref{ta:sfl}. Compared with the standard focal loss, the only additional parameter for the smoothed focal loss is $\sigma$ for normalizing the labels of anchors. By varying the $\sigma$ in Eqn.~\ref{eq:uk}, the smoothed focal loss can improve the performance from $39.2\%$ to $39.6\%$. Considering that most of detectors adopt one-hot labels for optimization, this experiment provides the evidence that the soft label can be more appropriate. It is also consistent with the process of label generation, where the value of IoU is not binary. 
\begin{table}[!ht]
\centering
\caption{Comparison of focal loss (FL) and the smoothed focal loss (SFL) as in Eqn.~\ref{eq:sfl}.}
\label{ta:sfl}
\begin{tabular}{ll"p{0.5cm}p{0.5cm}p{0.5cm}p{0.5cm}p{0.5cm}p{0.5cm}}
loss&$\sigma$&AP&AP$_{50}$&AP$_{75}$&AP$_{S}$&AP$_{M}$&AP$_{L}$\\\thickhline 
FL&&39.2&\textbf{57.5}&42.5&\textbf{22.7}&43.4&51.7\\
SFL &0.1&39.2&57.5&42.4&22.6&43.3&52.8\\
SFL &0.3&39.3&57.5&42.5&22.3&43.3&52.9\\
SFL&0.5&39.5&57.5&42.7&22.7&43.7&53.0\\
SFL&0.7&39.5&57.5&42.8&22.4&43.7&\textbf{53.3}\\
SFL&0.9&\textbf{39.6}&57.4&\textbf{43.1}&22.3&\textbf{43.8}&53.0
\end{tabular}
\label{table:sfloss}
\end{table}
  
\subsubsection{Inference Strategy} Finally, we evaluate different strategies in dealing with the outputs from anchor classification. Even with the anchor classifier, we may have multiple appropriate anchors at each location. The duplicated anchors can be eliminated by NMS operator as in the conventional pipeline. Considering that the anchor classifier is learned with the label from IoU after regression, we can keep a single anchor with the largest confidence for each location. It not only reduces the input size for NMS but also explores the supervised information from anchor regression more sufficiently. Table~\ref{ta:ct} compares two strategies. ``Pos'' denotes the conventional strategy that keeps all positive anchors for NMS, where $\tau$ is a threshold for the predicted probability. ``Top-$k$'' only adopts the top $k$ anchors with largest confidences for each location. Surprisingly, we find that Top-$1$ can be better than Pos by $0.1\%$ on AP and $0.4\%$ on AP$_{50}$. It is because keeping top $1$ anchor can reduce the noise in the input anchors for the NMS operator. This strategy is applied for the comparison to the stage-of-the-art detectors. 

\begin{table}[!ht]
\centering
\caption{Comparison of strategies for inference.}\label{ta:ct}
\begin{tabular}{ll"p{0.5cm}p{0.5cm}p{0.5cm}p{0.5cm}p{0.5cm}p{0.5cm}}
strategy&$\tau$&AP&AP$_{50}$&AP$_{75}$&AP$_{S}$&AP$_{M}$&AP$_{L}$\\\thickhline
Top-$1$&&\textbf{39.6}&\textbf{57.4}&43.1&22.3&43.8&\textbf{53.0}\\
Top-$2$&&39.5&57.2&43.2&22.3&43.8&53.0\\
Top-$5$&&39.5&57.1&\textbf{43.3}&\textbf{22.4}&\textbf{43.9}&53.0\\
Pos&0.1&39.5&56.8&43.3&22.3&43.9&53.0\\
Pos&0.2&39.5&57.0&43.3&22.4&43.9&53.0\\
Pos&0.5&38.8&55.9&42.7&21.2&43.4&52.7
\end{tabular}
\label{table:strategies}
\end{table}

\begin{table*}[!ht]
\centering
\caption{Comparison with state-of-the-art methods on COCO \textit{test-dev} set.}\label{ta:sota}
\begin{tabular}{l|l|lll|lll}
Methods&Backbone&AP&AP$_{50}$&AP$_{75}$&AP$_{S}$&AP$_{M}$&AP$_{L}$\\\thickhline
\textit{two-stage detectors}&&&&&&&\\
Faster R-CNN+++~\cite{HeZRS16}&ResNet-101-C4&34.9&55.7&37.4&15.6&38.7&50.9\\
Faster R-CNN w FPN~\cite{LinDGHHB17}&ResNet-101-FPN&36.2&59.1&39.0&18.2&39.0&48.2\\
Deformable R-FCN~\cite{DaiQXLZHW17}&Aligned-Inception-ResNet&37.5&58.0&40.8&19.4&40.1&52.5\\
Mask R-CNN~\cite{HeGDG17}&Resnet-101-FPN&38.2&60.3&41.7&20.1&41.1&50.2\\
Cascade R-CNN~\cite{CaiV18}&Resnet-101-FPN&42.8&62.1&46.3&23.7&45.5&55.2\\
\hline
\textit{one-stage detectors}&&&&&&&\\
YOLOv2~\cite{RedmonF17}&DarkNet-19&21.6&44.0&19.2&5.0&22.4&35.5\\
SSD513~\cite{LiuAESRFB16}&ResNet-101-SSD&31.2&50.4&33.3&10.2&34.5&49.8\\
DSSD513~\cite{FuLRTB17}&ResNet-101-DSSD&33.2&53.3&35.2&13.0&35.4&51.1\\
GA-RetinaNet~\cite{abs-1901-03278}&ResNet-50-FPN&37.1&56.9&40.0&20.1&40.1&48.0\\
RetinaNet~\cite{LinGGHD17}&ResNet-101-FPN&39.1&59.1&42.3&21.8&42.7&50.2\\
RetinaNet~\cite{LinGGHD17}&ResNeXt-32x8d-101-FPN&40.8&61.1&44.1&24.1&44.2&51.2\\
CornerNet~\cite{LawD18}&Hourglass-104&40.5&56.5&43.1&19.4&42.7&53.9 \\
CenterNet~\cite{abs-1904-08189}&Hourglass-104&44.9&62.4&48.1&25.6&47.4&\textbf{57.4} \\
FSAF~\cite{abs-1903-00621}&ResNet-101-FPN&40.9&61.5&44.0&24.0&44.2&51.3 \\
FSAF~\cite{abs-1903-00621}&ResNeXt-64x4d-101-FPN&42.9&63.8&46.3&27.0&47.9&52.7 \\
FCOS~\cite{tian2019fcos}&ResNet-101-FPN&41.0&60.7&44.1&24.0&44.1&51.0\\
FCOS~\cite{tian2019fcos}&ResNeXt-64x4d-101-FPN&43.2&62.8&46.6&26.5&46.2&53.3\\
FreeAnchor~\cite{abs-1909-02466}&ResNeXt-64x4d-101-FPN&44.8&64.3&48.4&27.6&47.5&56.0\\
\hline
Ours&ResNet-101-FPN&43.6&62.1&47.5&25.7&47.1&55.3\\
Ours&ResNeXt-32x8d-101-FPN&45.0&63.9&49.0&27.7&48.7&55.8\\
Ours&ResNeXt-64x4d-101-FPN&\textbf{45.4}&\textbf{64.3}&\textbf{49.4}&\textbf{27.8}&\textbf{49.0}&56.7
\end{tabular}
\label{table:result}
\end{table*}

\subsection{Comparison with State-of-the-Art}
In this subsection, we compare the proposed semi-anchored detector to the state-of-art detectors. All results in this subsection are evaluated on \textit{test-dev} set, where the public label is unavailable. Besides random horizontal flipping, we include scale jitter over scales  $\{640, 672, 704, 736, 768, 800\}$ as the additional augmentation to train a semi-anchored detector sufficiently. The number of epochs is increased to be 2$\times$ longer than that in Section~\ref{subsection:ablation}. Other settings remain the same. Table~\ref{ta:sota} summarizes the results of different methods. First, we observe that the performance of our method is significantly better than one-stage detectors with all anchors for classifications. Compared to RetinaNet, mAP is increased from $39.1\%$ to $43.6\%$, which gains more than $4\%$. It is because the classification problem in location classification is much balanced than that in the original anchor classification. Second, the performance of the semi-anchored detector also surpasses anchor-free algorithms, e.g., FCOS and FSAF, by a large margin. It implies that the strategy of assigning multiple anchors for each location is important for accurate regression. The promising performance on COCO data set illustrates that the proposed semi-anchored detector can benefit from both the anchor-free classification and the anchor-based regression. Moreover, our method outperforms two-stage detectors and is comparable to the multi-stage detector: Cascade R-CNN. It confirms the effectiveness of the proposed algorithm. Finally, we note that with a better backbone as ResNeXt64x4d-101-FPN~\cite{XieGDTH17}, the performance can be further improved and achieve $45.4\%$ mAP, which demonstrates a state-of-the-art performance for one-stage detectors.

\begin{table}[!ht]
\centering
\caption{Comparison of our semi-anchored (SA) strategy on SSD}\label{ta:ssd}
\begin{tabular}{ll"p{0.5cm}p{0.5cm}p{0.5cm}p{0.5cm}p{0.5cm}p{0.5cm}}
Methods&w/SA&AP&AP$_{50}$&AP$_{75}$&AP$_{S}$&AP$_{M}$&AP$_{L}$\\\thickhline
SSD300&&25.5&45.0&25.9&8.4&26.9&41.4\\
SSD300&\checkmark&27.5&43.1&29.0&11.2&31.2&43.4\\
SSD512&&30.1&51.3&31.3&12.9&34.2&43.6\\
SSD512&\checkmark&\textbf{32.0}&\textbf{48.8}&\textbf{34.3}&\textbf{15.9}&\textbf{37.9}&\textbf{46.5}
\end{tabular}
\end{table}

\subsection{Semi-Anchored SSD}
Besides RetinaNet, the proposed algorithm is easy to incorporate with existing detectors. To illustrate that, we embed semi-anchored strategy into SSD~\cite{LiuAESRFB16}. We adopt VGG16~\cite{SimonyanZ14a} as the backbone and train the model with 120 epochs. The results are shown in Table~\ref{ta:ssd}. With the proposed algorithm, the mAP is increased by $2.0\%$ at the scale of $300\times300$ and by $1.9\%$ at the scale of $500\times500$. The consistently improvement demonstrates that the proposed algorithm can work with other detection frameworks as well as RetinaNet and thus is flexible for real-world applications.

\section{Conclusion}
\label{sec:conclusion}
In this work, we develop a semi-anchored detector for one-stage object detection. Specifically, we propose to classify locations directly without anchors, which can mitigate the class-imbalance issue, but to keep anchors for regression, where anchor is essential for helping predict different shapes of bounding boxes. To bridge the gap between classification and regression task, we propose a new location labeling strategy using anchors and add a novel anchor classification head to refine the classification results on anchors at positive locations. The empirical study on COCO verifies that the proposed method can boost the performance of one-stage detector dramatically. This work provides some evidences for the effect of anchors in object detection. Further exploring the architecture of detectors with/without anchors can be our future work.

\bibliography{anchor19}

\begin{thebibliography}{36}
\providecommand{\natexlab}[1]{#1}
\providecommand{\url}[1]{\texttt{#1}}
\providecommand{\urlprefix}{URL }
\expandafter\ifx\csname urlstyle\endcsname\relax
  \providecommand{\doi}[1]{doi:\discretionary{}{}{}#1}\else
  \providecommand{\doi}{doi:\discretionary{}{}{}\begingroup
  \urlstyle{rm}\Url}\fi

\bibitem[{Cai and Vasconcelos(2018)}]{CaiV18}
Cai, Z.; and Vasconcelos, N. 2018.
\newblock Cascade {R-CNN:} Delving Into High Quality Object Detection.
\newblock In \emph{CVPR}, 6154--6162.

\bibitem[{Cao et~al.(2019)Cao, Pang, Han, and Li}]{CaoPHL19}
Cao, J.; Pang, Y.; Han, J.; and Li, X. 2019.
\newblock Hierarchical Shot Detector.
\newblock In \emph{ICCV}, 9704--9713. {IEEE}.

\bibitem[{Dai et~al.(2016)Dai, Li, He, and Sun}]{DaiLHS16}
Dai, J.; Li, Y.; He, K.; and Sun, J. 2016.
\newblock {R-FCN:} Object Detection via Region-based Fully Convolutional
  Networks.
\newblock In \emph{NIPS}, 379--387.

\bibitem[{Dai et~al.(2017)Dai, Qi, Xiong, Li, Zhang, Hu, and Wei}]{DaiQXLZHW17}
Dai, J.; Qi, H.; Xiong, Y.; Li, Y.; Zhang, G.; Hu, H.; and Wei, Y. 2017.
\newblock Deformable Convolutional Networks.
\newblock In \emph{ICCV}, 764--773.

\bibitem[{Duan et~al.(2019)Duan, Bai, Xie, Qi, Huang, and
  Tian}]{abs-1904-08189}
Duan, K.; Bai, S.; Xie, L.; Qi, H.; Huang, Q.; and Tian, Q. 2019.
\newblock CenterNet: Keypoint Triplets for Object Detection.
\newblock \emph{CoRR} abs/1904.08189.

\bibitem[{Fu et~al.(2017)Fu, Liu, Ranga, Tyagi, and Berg}]{FuLRTB17}
Fu, C.; Liu, W.; Ranga, A.; Tyagi, A.; and Berg, A.~C. 2017.
\newblock {DSSD} : Deconvolutional Single Shot Detector.
\newblock \emph{CoRR} abs/1701.06659.

\bibitem[{Girshick(2015)}]{Girshick15}
Girshick, R.~B. 2015.
\newblock Fast {R-CNN}.
\newblock In \emph{ICCV}, 1440--1448.

\bibitem[{Girshick et~al.(2014)Girshick, Donahue, Darrell, and
  Malik}]{GirshickDDM14}
Girshick, R.~B.; Donahue, J.; Darrell, T.; and Malik, J. 2014.
\newblock Rich Feature Hierarchies for Accurate Object Detection and Semantic
  Segmentation.
\newblock In \emph{CVPR}, 580--587.

\bibitem[{Gkioxari, Malik, and Johnson(2019)}]{abs-1906-02739}
Gkioxari, G.; Malik, J.; and Johnson, J. 2019.
\newblock Mesh {R-CNN}.
\newblock \emph{CoRR} abs/1906.02739.

\bibitem[{He et~al.(2017)He, Gkioxari, Doll{\'{a}}r, and Girshick}]{HeGDG17}
He, K.; Gkioxari, G.; Doll{\'{a}}r, P.; and Girshick, R.~B. 2017.
\newblock Mask {R-CNN}.
\newblock In \emph{ICCV}, 2980--2988.

\bibitem[{He et~al.(2016)He, Zhang, Ren, and Sun}]{HeZRS16}
He, K.; Zhang, X.; Ren, S.; and Sun, J. 2016.
\newblock Deep Residual Learning for Image Recognition.
\newblock In \emph{CVPR}, 770--778.

\bibitem[{Hinton, Vinyals, and Dean(2015)}]{HintonVD15}
Hinton, G.~E.; Vinyals, O.; and Dean, J. 2015.
\newblock Distilling the Knowledge in a Neural Network.
\newblock \emph{CoRR} abs/1503.02531.

\bibitem[{Kong et~al.(2019)Kong, Sun, Liu, Jiang, and Shi}]{06563}
Kong, T.; Sun, F.; Liu, H.; Jiang, Y.; and Shi, J. 2019.
\newblock Consistent Optimization for Single-Shot Object Detection.
\newblock \emph{CoRR} abs/1901.06563.

\bibitem[{Law and Deng(2018)}]{LawD18}
Law, H.; and Deng, J. 2018.
\newblock CornerNet: Detecting Objects as Paired Keypoints.
\newblock In \emph{ECCV}, 765--781.

\bibitem[{Lin et~al.(2017{\natexlab{a}})Lin, Doll{\'{a}}r, Girshick, He,
  Hariharan, and Belongie}]{LinDGHHB17}
Lin, T.; Doll{\'{a}}r, P.; Girshick, R.~B.; He, K.; Hariharan, B.; and
  Belongie, S.~J. 2017{\natexlab{a}}.
\newblock Feature Pyramid Networks for Object Detection.
\newblock In \emph{CVPR}, 936--944.

\bibitem[{Lin et~al.(2017{\natexlab{b}})Lin, Goyal, Girshick, He, and
  Doll{\'{a}}r}]{LinGGHD17}
Lin, T.; Goyal, P.; Girshick, R.~B.; He, K.; and Doll{\'{a}}r, P.
  2017{\natexlab{b}}.
\newblock Focal Loss for Dense Object Detection.
\newblock In \emph{ICCV}, 2999--3007.

\bibitem[{Lin et~al.(2014)Lin, Maire, Belongie, Hays, Perona, Ramanan,
  Doll{\'{a}}r, and Zitnick}]{LinMBHPRDZ14}
Lin, T.; Maire, M.; Belongie, S.~J.; Hays, J.; Perona, P.; Ramanan, D.;
  Doll{\'{a}}r, P.; and Zitnick, C.~L. 2014.
\newblock Microsoft {COCO:} Common Objects in Context.
\newblock In \emph{ECCV}, 740--755.

\bibitem[{Liu et~al.(2016)Liu, Anguelov, Erhan, Szegedy, Reed, Fu, and
  Berg}]{LiuAESRFB16}
Liu, W.; Anguelov, D.; Erhan, D.; Szegedy, C.; Reed, S.~E.; Fu, C.; and Berg,
  A.~C. 2016.
\newblock {SSD:} Single Shot MultiBox Detector.
\newblock In \emph{ECCV}, 21--37.

\bibitem[{Lu et~al.(2018)Lu, Li, Yue, Li, and Yan}]{abs-1811-12030}
Lu, X.; Li, B.; Yue, Y.; Li, Q.; and Yan, J. 2018.
\newblock Grid {R-CNN}.
\newblock \emph{CoRR} abs/1811.12030.

\bibitem[{Pang et~al.(2019)Pang, Chen, Shi, Feng, Ouyang, and
  Lin}]{PangCSFOL19}
Pang, J.; Chen, K.; Shi, J.; Feng, H.; Ouyang, W.; and Lin, D. 2019.
\newblock Libra {R-CNN:} Towards Balanced Learning for Object Detection.
\newblock In \emph{CVPR}, 821--830.

\bibitem[{Qian et~al.(2020)Qian, Chen, Li, and Jin}]{QianCLJ20}
Qian, Q.; Chen, L.; Li, H.; and Jin, R. 2020.
\newblock {DR} Loss: Improving Object Detection by Distributional Ranking.
\newblock In \emph{CVPR}, 12161--12169. {IEEE}.

\bibitem[{Redmon and Farhadi(2017)}]{RedmonF17}
Redmon, J.; and Farhadi, A. 2017.
\newblock {YOLO9000:} Better, Faster, Stronger.
\newblock In \emph{CVPR}, 6517--6525.

\bibitem[{Ren et~al.(2015)Ren, He, Girshick, and Sun}]{RenHGS15}
Ren, S.; He, K.; Girshick, R.~B.; and Sun, J. 2015.
\newblock Faster {R-CNN:} Towards Real-Time Object Detection with Region
  Proposal Networks.
\newblock In \emph{NIPS}, 91--99.

\bibitem[{Russakovsky et~al.(2015)Russakovsky, Deng, Su, Krause, Satheesh, Ma,
  Huang, Karpathy, Khosla, Bernstein, Berg, and Fei-Fei}]{ILSVRC15}
Russakovsky, O.; Deng, J.; Su, H.; Krause, J.; Satheesh, S.; Ma, S.; Huang, Z.;
  Karpathy, A.; Khosla, A.; Bernstein, M.; Berg, A.~C.; and Fei-Fei, L. 2015.
\newblock {ImageNet Large Scale Visual Recognition Challenge}.
\newblock \emph{International Journal of Computer Vision} .

\bibitem[{Simonyan and Zisserman(2015)}]{SimonyanZ14a}
Simonyan, K.; and Zisserman, A. 2015.
\newblock Very Deep Convolutional Networks for Large-Scale Image Recognition.
\newblock In \emph{ICLR}.

\bibitem[{Tian et~al.(2019)Tian, Shen, Chen, and He}]{tian2019fcos}
Tian, Z.; Shen, C.; Chen, H.; and He, T. 2019.
\newblock {FCOS}: Fully Convolutional One-Stage Object Detection.
\newblock In \emph{ICCV}.

\bibitem[{Uijlings et~al.(2013)Uijlings, van~de Sande, Gevers, and
  Smeulders}]{UijlingsSGS13}
Uijlings, J. R.~R.; van~de Sande, K. E.~A.; Gevers, T.; and Smeulders, A. W.~M.
  2013.
\newblock Selective Search for Object Recognition.
\newblock \emph{International Journal of Computer Vision} 104(2): 154--171.

\bibitem[{Vu et~al.(2019)Vu, Jang, Pham, and Yoo}]{VuJPY19}
Vu, T.; Jang, H.; Pham, T.~X.; and Yoo, C.~D. 2019.
\newblock Cascade {RPN:} Delving into High-Quality Region Proposal Network with
  Adaptive Convolution.
\newblock In Wallach, H.~M.; Larochelle, H.; Beygelzimer, A.;
  d'Alch{\'{e}}{-}Buc, F.; Fox, E.~B.; and Garnett, R., eds., \emph{NeurIPS},
  1430--1440.

\bibitem[{Wang et~al.(2019)Wang, Chen, Yang, Loy, and Lin}]{abs-1901-03278}
Wang, J.; Chen, K.; Yang, S.; Loy, C.~C.; and Lin, D. 2019.
\newblock Region Proposal by Guided Anchoring.
\newblock \emph{CoRR} abs/1901.03278.

\bibitem[{Wu and He(2018)}]{WuH18}
Wu, Y.; and He, K. 2018.
\newblock Group Normalization.
\newblock In \emph{ECCV}, 3--19.

\bibitem[{Xie et~al.(2017)Xie, Girshick, Doll{\'{a}}r, Tu, and He}]{XieGDTH17}
Xie, S.; Girshick, R.~B.; Doll{\'{a}}r, P.; Tu, Z.; and He, K. 2017.
\newblock Aggregated Residual Transformations for Deep Neural Networks.
\newblock In \emph{CVPR}, 5987--5995.

\bibitem[{Yang et~al.(2018)Yang, Zhang, Li, Zhang, and Sun}]{YangZLZS18}
Yang, T.; Zhang, X.; Li, Z.; Zhang, W.; and Sun, J. 2018.
\newblock MetaAnchor: Learning to Detect Objects with Customized Anchors.
\newblock In \emph{NIPS}, 318--328.

\bibitem[{Yu et~al.(2016)Yu, Jiang, Wang, Cao, and Huang}]{YuJWCH16}
Yu, J.; Jiang, Y.; Wang, Z.; Cao, Z.; and Huang, T.~S. 2016.
\newblock UnitBox: An Advanced Object Detection Network.
\newblock In \emph{ACMMM}, 516--520.

\bibitem[{Zhang et~al.(2020)Zhang, Chi, Yao, Lei, and Li}]{ZhangCYLL20}
Zhang, S.; Chi, C.; Yao, Y.; Lei, Z.; and Li, S.~Z. 2020.
\newblock Bridging the Gap Between Anchor-Based and Anchor-Free Detection via
  Adaptive Training Sample Selection.
\newblock In \emph{CVPR}, 9756--9765. {IEEE}.

\bibitem[{Zhang et~al.(2019)Zhang, Wan, Liu, Ji, and Ye}]{abs-1909-02466}
Zhang, X.; Wan, F.; Liu, C.; Ji, R.; and Ye, Q. 2019.
\newblock FreeAnchor: Learning to Match Anchors for Visual Object Detection.
\newblock \emph{CoRR} abs/1909.02466.

\bibitem[{Zhu, He, and Savvides(2019)}]{abs-1903-00621}
Zhu, C.; He, Y.; and Savvides, M. 2019.
\newblock Feature Selective Anchor-Free Module for Single-Shot Object
  Detection.
\newblock \emph{CoRR} abs/1903.00621.

\end{thebibliography}

\end{document}